\documentclass[letterpaper, 10 pt, conference]{ieeeconf}  

\IEEEoverridecommandlockouts                              

\usepackage{cite}
\usepackage{amsmath, amssymb, url}
\usepackage{bm,xpatch}
\usepackage[capitalize]{cleveref}
\usepackage{xcolor}
\usepackage{fp, tikz, pgfplots}
\usepackage{siunitx}
\usepackage{algorithm}
\usepackage{algorithmic}
\usepackage{mathtools}
\usepackage{multirow}
\usepackage{booktabs}
\usepackage{interval}
\usepackage{capt-of}
\usepackage[usestackEOL]{stackengine}
\usepackage{array}
\usepackage{multicol}
\usepackage{wrapfig}
\usetikzlibrary{arrows}
\usetikzlibrary{shapes}
\usetikzlibrary{patterns}
\usetikzlibrary{decorations.pathmorphing}
\usetikzlibrary{decorations.pathreplacing}
\usetikzlibrary{decorations.shapes}
\usetikzlibrary{decorations.text}
\usetikzlibrary{shapes.misc}
\usetikzlibrary{decorations.markings}
\usetikzlibrary{arrows.meta}
\usetikzlibrary{backgrounds}
\usepgfplotslibrary{fillbetween}
\usepgfplotslibrary{statistics}

\usepackage{tikz}
\usepackage{tikz-cd}
\usetikzlibrary{shapes,arrows,fit,backgrounds,positioning}
\usetikzlibrary{decorations.pathmorphing} 
\usetikzlibrary{matrix} 
\usetikzlibrary{arrows} 
\usetikzlibrary{calc} 
\usepackage{verbatim}

\newtheorem{definition}{Definition}
\newtheorem{assumption}{Assumption}

\newtheorem{lemma}{Lemma}
\newtheorem{theorem}{Theorem}
\newtheorem{corollary}{Corollary}
\newtheorem{proposition}{Proposition}
\newtheorem{remark}{Remark}

\makeatletter
\def\blfootnote{\gdef\@thefnmark{}\@footnotetext}
\makeatother

\definecolor{tumblue}{HTML}{0065bd} 
\definecolor{tumblue4}{HTML}{98c6ea} 
\definecolor{tumbluered}{HTML}{bd6500}
\definecolor{tumblue4red}{HTML}{eac698}
\newcommand{\icol}[1]{
  \left(\begin{smallmatrix}#1\end{smallmatrix}\right)%
}

\newcommand\clr[2]{{\color{#1}{#2}}}
\newcommand\ko[0]{{\mathcal{K}}}

\newcommand\Set[1]{\mathbb{#1}} 
\newcommand{\tsgn}[1]{{#1}}

\pgfplotsset{width=10\columnwidth /10, compat = 1.13, 
	height = 55\columnwidth /100, grid= major, 
	legend cell align = left, ticklabel style = {font=\scriptsize},
	every axis label/.append style={font=\small},
	legend style = {font=\tiny},title style={yshift=-7pt, font = \small} }

\makeatletter
\tikzset{
    dot diameter/.store in=\dot@diameter,
    dot diameter=3pt,
    dot spacing/.store in=\dot@spacing,
    dot spacing=3pt,
    dots/.style={
        line width=\dot@diameter,
        line cap=round,
        dash pattern=on 0pt off \dot@spacing
    }
}
\makeatother

\title{\LARGE \bf
Diffeomorphically Learning Stable Koopman Operators
}

\author{{Petar Bevanda}$^{1}$, {Max Beier}$^{1}$, {Sebastian Kerz}$^{1}$, {Armin Lederer}$^{1}$, {Stefan Sosnowski}$^{1}$ and Sandra Hirche$^{1}$
\thanks{*This work was supported by the European Union's Horizon 2020 research and innovation programme under grant agreement no. 871295 "SeaClear".
}
\thanks{$^{1}$The authors are with the Department of Electrical and Computer Engineering, Technical University of Munich, Germany.
{\tt\small\{petar.bevanda, max.beier, s.kerz, armin.lederer, sosnowski, hirche\}@tum.de}.}%
}

\begin{document}

\maketitle
\thispagestyle{empty}
\pagestyle{empty}

\begin{abstract}
\clr{black}{System representations inspired by the infinite-dimensional Koopman operator (generator) are increasingly considered for predictive modeling. Due to the operator's linearity, a range of nonlinear systems admit linear predictor representations -- allowing for simplified prediction, analysis and control%
}. However, finding meaningful finite-dimensional representations for prediction is difficult as it involves {determining} features that are both {Koopman-invariant} (evolve linearly under the dynamics) as well as {relevant} (spanning the original state) -- a generally unsupervised problem. 
 In this work, we present \textit{Koopmanizing Flows} -- a novel \clr{black}{continuous-time} framework for supervised learning of linear predictors for a class of nonlinear dynamics. \clr{black}{In our model construction a latent diffeomorphically related linear system unfolds into a linear predictor through the composition with a monomial basis. The lifting, its linear dynamics and state reconstruction are learned simultaneously, while an unconstrained parameterization of Hurwitz matrices ensures asymptotic stability regardless of the operator approximation accuracy.}
The superior efficacy of Koopmanizing Flows is demonstrated in comparison to a state-of-the-art method on the well-known LASA handwriting benchmark.
\end{abstract}
\section{INTRODUCTION}
Global linearization methods for nonlinear systems inspired by the infinite-dimensional, \textit{linear} Koopman operator~\cite{Koopman1931} have received increased attention for modeling nonlinear dynamics in recent years \cite{Korda2020b,Lian2019,Bevanda2021b,Fan2021}. 
\clr{black}{Compared to conventional state-space modeling, lifting a finite-dimensional nonlinear system to a higher-dimensional linear operator representation allows for simplified, linear predictor, models that are compatible with linear techniques for prediction, analysis and control~\cite{Bevanda2021a}.}

However, obtaining long-term accurate predictive models using finite-dimensional Koopman operator dynamical models is challenging,
\clr{black}{~as it generally incorporates an unsupervised learning problem. The latter involves learning a linear predictor whose coordinates are both \textit{Koopman-invariant}, i.e., their evolution remains in the span of the features, as well as \textit{relevant} enough to (almost) fully span the original state -- reconstructing it in a linear fashion.

As solving the unsupervised problem is challenging, the majority of works make the problem supervised by projecting the operator onto predetermined features -- akin to Galerkin methods -- using well-known EDMD \cite{Williams2015a}.} 
However, presupposing a suitable basis of functions is a very strong assumption for linear time-invariant prediction -- leading to only locally accurate models.
Other approaches learn the features simultaneously \cite{Li2017a} or in a decoupled manner \cite{Lian2019} leveraging the expressive power of neural networks or kernel methods, but often lack theoretical justification. 

To learn the relevant features and operator spectrum simultaneously using machine learning, the sole expressivity of the learning methods does not immediately lead to reliability in solving the generally unsupervised problem of learning Koopman-related models as it requires certain structure to be well-posed. For dissipative systems, a way to improve the reliability of learning these models is enforcing stability. \clr{black}{Nonetheless, few approaches impose such a constraint with guarantees. The authors in \cite{Pan2020} show improved performance by a stable transition matrix in the latent space. However, they only parameterize diagonalizable matrices without considering the stability of the resulting model as a whole.
The recent SKEL framework \cite{Fan2021} provides asymptotic stability guarantees of the learned model in a fully data-driven manner. Notably, the former and latter approaches do not reconstruct the observable of interest linearly -- prohibiting advantageous reformulations for prediction, control and estimation using tools from linear systems theory. Furthermore, they are dependent on trajectories, whose length impacts performance as their optimization objective minimizes a multi-step error in lifted space. 
As a consequence, there are no supervised targets to fit as they comprise of maps that themselves need to be optimized, making the problem unsupervised.

We, however, consider the continuous-time setting, as it is native to many physical and biological systems. Constructing linear predictors in continuous-time allows for trajectory-independent learning and provides a valid model for arbitrary discretization times.
The related work of \cite{Bevanda2021b} considers the continuous-time setting but is not fully data-driven as it assumes a known Jacobian linearization diffeomorphic to the nonlinear system -- often a strong assumption. The reconstruction matrix is subsequently fitted without consideration of generalized eigenspaces.
}

\clr{black}{
In this paper we present \textit{Koopmanizing Flows}, a fully data-driven framework for learning asymptotically stable continuous-time \textit{linear predictors} that ensures system- and Koopman-theoretic considerations are embedded in the learning approach. 
The Koopman-theoretic aspect considers a lifting construction that preserves Koopman-invariance, whereas the system-theoretic notions are related to stability and smooth equivalence.
These two aspects merge in our model construction as a latent diffeomorphically related system expands into a linear predictor through the composition with a monomial basis.
The lifting, its linear dynamics and state reconstruction are learned simultaneously in a supervised fashion while an unconstrained parameterization of stable matrices ensures asymptotic stability regardless of the operator approximation accuracy.}
To the best of our knowledge, this is the first trajectory-independent, continuous-time, framework that learns provably stable linear predictors for nonlinear systems. We demonstrate the superior performance of the proposed method in comparison to a state-of-the-art method on the well-known LASA handwriting benchmark.

This paper is structured as follows. After the problem statement in Section \ref{sec:KOthry}, we present a novel data-driven framework -- Koopmanizing Flows -- for constructing stable, Koopman operator dynamical models in Section \ref{sec:KF} which is followed by an evaluation in Section \ref{sec:Eval} and a conclusion.
\section{PROBLEM STATEMENT}\label{sec:KOthry}
 Consider an unknown, continuous-time nonlinear dynamical system\footnote{\textbf{Notation:}
		Lower/upper case bold symbols $\bm{x}$/$\bm{X}$ denote vectors/matrices. Symbols $\mathbb{N}/\mathbb{ R }/\mathbb{C}$ denote sets of natural/real/complex numbers, while $\mathbb{N}_{0}$ denotes all natural numbers with zero, and $\mathbb{R}_{+,0}/\mathbb{R}_{+}$ all positive reals with/without zero. 
		Function spaces with a specific integrability/regularity order are denoted as $L^{}$/$C^{}$ with the order in their exponent. The Jacobian matrix of vector-valued map $\bm{\psi}$ evaluated at $\bm{x}$ is denoted as $\bm{J}_{\bm{\psi}}(\bm{x})$.
The $L^p$-norm on a set $\Set{X}$ is denoted as $\|\cdot\|_{p, \Set{X}}$. Writing $\odot$ denotes the Hadamard product, $\operatorname{exp}$ pointwise exponential and $\circ$ function composition.}
	\begin{equation}\label{eq:sys}
	\dot{\bm{x}}=\bm{f}(\bm{x})
	\end{equation}
 with continuous states $\bm{x}\in\Set{X}$ on a compact set $\Set{X} \subset \mathbb{R}^{d}$ such that $\bm{f} \in C^2(\Set{X})$.
\begin{assumption}\label{ass:sysCLS}
{We assume the dynamical system \eqref{eq:sys} has a globally exponentially stable origin.}
\end{assumption}
The assumption is fulfilled for a fairly large class of practically relevant dynamics including, e.g., human reaching movements \cite{Khansari2011} or dissipative Lagrangian systems such as neutrally buoyant underwater vehicles~\cite{bookFOSS}.

Due to their continuous-time nature, the dynamics are fully described by the forward-complete flow map~\cite{Bittracher2015} of \eqref{eq:sys} given by \clr{black}{$\bm{F}^{t}(\bm{x}_0):= \bm{x}_0+\int_{t_{0}}^{t_{0}+t} \bm{f}(\bm{x}(\tau)) d \tau$}
which has a unique solution on $[0,+\infty)$ from the initial
condition $\bm{x}$ at $t = 0$ due to stability of the isolated attractor~\cite{Angeli1999}. This flow map naturally induces the associated Koopman operator semigroup as defined in the following.
\begin{definition}\label{def:Koop}
	The semigroup of Koopman operators $\{{\mathcal{K}}^{t}\}_{t \in \mathbb{R}_{+,0}}\!\!:C(\Set{X}) \!\mapsto\! C(\Set{X})$
	acts on an observable function ${h} \!\in\! C(\Set{X})$ on the state-space $\Set{X}$ through ${\mathcal{K}^{t}_{\bm{f}}} {{h}} ={{h}}\circ{\bm{F}^{t}}$.
\end{definition}
In simple terms, the operator applied to an observable function $h$ at time $t_0$ moves it along the solution curves if \eqref{eq:sys} as $\ko^t_{\bm{f}}{h}(\bm{x}(t_0))={h}(\bm{x}(t_0+t))$. Applied component-wise to the identity observable $\bm{h}(\bm{x})=\bm{x}$, it equals the flow $\bm{F}_{t}(x(t_0))$. Crucially, every $\ko^{t}_{\bm{f}}$ is a \textit{linear}\footnote{Consider $h_{1}, h_{2} \in {C}(\Set{X})$ and $\beta \in \Set{C}$. Then, using Definition \ref{def:Koop}, $\mathcal{K}_{t}\left(\beta h_{1}\!+\!h_{2}\right)=\left(\beta h_{1}\!+\!h_{2}\right) \circ \bm{F}_{t}=\beta h_{1} \circ \bm{F}_{t}\!+\!h_{2} \circ \bm{F}_{t}=\beta \mathcal{K}_{t} h_{1}\!+\!\mathcal{K}_{t} h_{2}$.} operator.
With a well-defined Koopman operator semigroup, we introduce its infinitesimal generator.\looseness=-1
\clr{black}{\begin{definition}[{\cite{Lasota1994}}]\label{def:generator}
The linear operator $\mathcal{G}_{\ko_{\bm{f}}}$ fulfilling $\mathcal{G}_{{\mathcal{K}_{\bm{f}}}} {h}=\lim _{t \rightarrow 0^{+}} {({\mathcal{K}}^{t} {h}-{h}})/{t} = \dot{h}$
	is the infinitesimal generator of the semigroup of Koopman operators $\{{\mathcal{K}}^{t}\}_{t \in \mathbb{R}_{+,0}}$.
\end{definition}}
The strength of the Koopman operator formalism is that it allows to decompose dynamics into linearly evolving coordinates, which naturally arise through the eigenfunctions of the evolution operator $\mathcal{G}_{{\mathcal{K}_{\bm{f}}}}$. These eigenfunctions are formally defined as follows.
\begin{definition}\label{eigF}
	An observable $\phi \in C(\Set{X})$ is called an \textit{eigenfunction} of $\mathcal{G}_{{\mathcal{K}_{\bm{f}}}}$ if it satisfies $\mathcal{G}_{\mathcal{K}_{\bm{f}}} \phi = \lambda \phi$,
	for an \textit{eigenvalue} $\lambda \in \mathbb{C}$. The span of eigenfunctions $\phi$ of $\mathcal{G}_{{\mathcal{K}_{\bm{f}}}}$ is denoted by~$\bm{\Phi}$.\looseness=-1
\end{definition}

With the above definitions, it is evident that the Koopman operator theory is inherently tied to the temporal evolution of dynamical systems \cite{Bevanda2021a}. Moreover, due to Assumption~\ref{ass:sysCLS}, 
the Koopman operator generator has a pure point spectrum for the dynamics \eqref{eq:sys} \cite{Mauroy2016b}. 
Therefore, for each observable $h$, there exists a sequence $v_j(h)\in\Set{C}$ of mode weights, such that we obtain the decomposition 
\begin{equation}\label{eq:KMD}
\dot{h} =\mathcal{G}_{\ko_{\bm{f}}} h =\mathcal{G}_{\ko_{\bm{f}}}\sum_{j=1}^{\infty} v_{j}(h) {\phi}_{j}=\sum_{j=1}^{\infty} v_{j}(h) \lambda_{j} {\phi}_{j}, 
\end{equation}
which is a superposition of infinitely many linear ODEs. With a slight abuse of operator notation, we can write the decomposition \eqref{eq:KMD} compactly as $\dot{h}=\mathcal{V}_{h}\mathcal{G}_{\ko_{\bm{f}}}\bm{\Phi}$, \clr{black}{where $\mathcal{V}_{h}$ is an operator projecting $\bm{\Phi}$ on the observable of interest.}
{Given the above, we introduce a \textit{generalized}
description of Koopman-invariant coordinates. 
\begin{definition}\label{def:LTIbase}
Consider the system \eqref{eq:sys}, a matrix $\bm{{A}} \in \Set{R}^{D \times D}$ and a finite collection of features $\bm{\psi}:=\left[\psi_{1}(\bm{x}), \ldots, \psi_{D}(\bm{x})\right]^{\top}$ with $\psi_i(\bm{x}) \in C^{1}(\Set{X})$ on a compact set $\mathbb{X}$. If $\bm{\psi}$ satisfies 
$\mathcal{G}_{\ko_{\bm{f}}}\bm{\psi}(\bm{x}):=\bm{J}_{\bm{\psi}}(\bm{x})\bm{f}(\bm{x})=\bm{\bm{{A}}} \bm{\psi}(\bm{x})$
it represents \textit{Koopman-invariant} features/coordinates for \eqref{eq:sys}.
\end{definition}
The above definition helps us pose the following functional optimization problem
\looseness=-1
\begin{subequations}\label{opt_probPre}
\begin{align}
\operatornamewithlimits{min}_{\bm{{A}} \in \Set{R}^{D \times D}, \bm{C} \in \Set{R}^{d \times D}, \bm{\psi}(\cdot)} &{\overbrace{\|\dot{\bm{h}}-\bm{C}\bm{A}\boldsymbol{\psi}\|}^{\text{prediction}}+ \overbrace{\|\bm{h}-\bm{C}\bm{\psi}\|}^{\text{reconstruction}}} \label{obj_optPre} \\
\text{s.t. \quad}  \bm{J}_{\bm{\psi}}(\bm{x})\bm{f}(\bm{x})&=\bm{\bm{{A}}} \bm{\psi}(\bm{x}) \label{space_optPre} \\
 \bm{A}\text{~is}&\text{~Hurwitz}\label{stab_cnstrPre}
\end{align}
\end{subequations}
\looseness=-1 
for obtaining a finite-dimensional model of~\eqref{eq:KMD}, e.g., for the full-state $\bm{h}(\bm{x})\tsgn{=}\operatorname{id}(\bm{x})$, where \eqref{space_optPre} ensures Koopman-invariance and~\eqref{obj_optPre} minimizes the prediction and projection errors of the feature collection onto the observable of interest. Note that the Hurwitz condition \eqref{stab_cnstrPre} suffices\footnote{\clr{black}{This is true under certain conditions \cite[Proposition 1, Remark 2]{Yi2021}.}} to ensure the asymptotic stability of the feature dynamics. 
As finding an analytical solution to \eqref{opt_probPre} is generally not feasible even in the case of known dynamics $\bm{f}(\cdot)$, we use data samples in order to obtain one. 
\clr{black}{To allow for a multi-variate regression problem formulation, we define the target vector ${{\bm{\xi}}} \equiv [{{\bm{\xi}_1}}{}^{\top}~\bm{\xi}_2{}^{\top}]{}^{\top}$, $\bm{\xi}_1=\dot{\bm{x}}$, $\bm{\xi}_2=\bm{x}$ for finding a function $\tilde{\bm{f}}$ such that $\bm{\xi}=\tilde{\bm{f}}(\bm{x})$.
\begin{assumption}\label{ass:data}
 A dataset of $N$ input-output pairs $\mathbb{D}_{N}=\{\bm{x}^{(i)}, {{\bm{\xi}}}^{(i)}\equiv\tilde{\bm{f}}(\bm{x}^{(i)})\}_{i=1}^{N}$ for ~\eqref{eq:sys} 
is available.
\end{assumption}}
Having measurements of the state and its time-derivative at disposal is a common assumption. Note that we do not require the dataset to reflect one or multiple trajectories. If not directly available, the time-derivative of the state can be approximated through finite differences for practical applications. Based on the finite dataset from Assumption \ref{ass:data}, we consider the following sample-based approximation of the optimization problem \eqref{opt_probPre}
\begin{subequations}\label{opt_prob}
\clr{black}{
\begin{align}
\operatornamewithlimits{min}_{\substack{\bm{{A}}, \bm{C}, \bm{\psi}(\cdot)}}  &\sum^{N}_{i=1}
{\left\|{{\bm{\xi}}}^{(i)}\tsgn{-}\begin{bmatrix}\bm{C}\bm{A}\\\bm{C}\end{bmatrix}\boldsymbol{\psi}(\bm{x}^{(i)})\right\|^{2}_2}\label{obj_opt} \\
\text{s.t.} \quad &\bm{J}_{\bm{\psi}}(\bm{x}^{(i)}){{\bm{\xi}}}_{1}^{(i)}=\bm{{{A}}} \bm{\psi}(\bm{x}^{(i)})\label{space_opt} \\
& \quad \bm{A}  \text{ is Hurwitz}\label{stab_cnstr}
\end{align}
}
\end{subequations}
delivering a finite-dimensional \clr{black}{{linear predictor}
}
\begin{subequations}\label{eq:LTI full}
\begin{align}
\bm{z}^{}_0 &=\bm{\psi}^{}(\bm{x}(0)), \label{eq:LTI:1}\\
\dot{\bm{z}}^{} &=\bm{{A}}\bm{z}^{}, \label{eq:LTI:2}\\
\hat{\bm{x}} &=\bm{C} \bm{z}^{} \label{eq:LTI:3}
\end{align}
\end{subequations}
as a representation of the Koopman operator generator.}
With this model, the nonlinearity of a $d$-dimensional ODE (\ref{eq:sys}) is traded for a nonlinear ``lift" (\ref{eq:LTI:1}) of the initial condition $\bm{x}(0)$ to higher dimensional ($D \gg d$) Koopman-invariant coordinates \eqref{eq:LTI:2} such that the original state can be linearly reconstructed via \eqref{eq:LTI:3}. 
Moreover, \eqref{opt_prob} {allows} to identify an arbitrary amount of Koopman-invariant features directly instead of only finding ones that lie in a heuristically predetermined dictionary of functions.
Thus, the sole error source in the resulting system \eqref{eq:LTI full} is due to the finite truncation of the infinite sum in \eqref{eq:KMD}. This is crucial for long-term accurate linear prediction, when, e.g., the model \eqref{eq:LTI full} is used as a motion generator \cite{Khansari2011} under safety-critical operation limits. 
\section{DIFFEOMORPHICALLY LEARNING KOOPMAN-INVARIANT COORDINATES}\label{sec:KF}
\clr{black}{Representing linear and nonlinear systems with equilibria differs solely in the fact that, in the case of linear systems, \clr{black}{the expansion \eqref{eq:KMD}, e.g., $\bm{h}(\bm{x})\tsgn{=}\operatorname{id}(\bm{x})$ is finite}, while in the nonlinear case it is generally infinite \cite{Mezic2019}. Nevertheless, a finite-dimensional linear system can be lifted to infinite-dimensions through generalized monomial features known to preserve Koopman-invariance \cite{Zeng2018}. We exploit the aforementioned property of linear systems to construct a generalized expansion of \eqref{eq:KMD} \clr{black}{for nonlinear systems admitting an exact linearization via a diffeomorphic coordinate change.}
As illustrated in Figure \ref{fig:commDiag}, the idea is to ``morph'' the original nonlinear dynamics into latent linear dynamics via a diffeomorphism $\bm{d}$. Nevertheless, finite-dimension convergence results for a decomposition in the form of \eqref{eq:KMD} are an open research question~\cite{Bevanda2021a} and out of scope for this work.}
\subsection{Construction of Lifting Functions}

Instead of directly attempting to find solutions per Definition \ref{def:LTIbase}, we use a diffeomorphic relation to a latent linear model to obtain those, providing us with Koopman-invariant lifting coordinates that fulfill \eqref{space_opt}.
\begin{definition}\label{def:smthEQ}
    Vector fields $\dot{\bm{x}}=\bm{f}(\bm{x})$ and $\dot{\bm{y}}=\bm{g}(\bm{y})$ are diffeomorphic, or smoothly equivalent, if there exists a diffeomorphism $\bm{d}\!: \mathbb{R}^{d} \!\mapsto\! \mathbb{R}^{d}$ such that $ \bm{f}(\bm{x})\!=\!\bm{J}^{-1}_{\bm{d}}(\bm{x})\bm{g}(\bm{d}(\bm{x}))$.
\end{definition}

In essence, diffeomorphic systems have equivalent dynamics just in different coordinates. This is what we exploit as the system \eqref{eq:sys} is diffeomorphic to a latent linear system $\dot{\bm{y}}=\underline{\bm{A}}\bm{y}$ under Assumption \ref{ass:sysCLS} \cite{Lan2013}.
\clr{black}{It is straightforward to see that the diffeomorphism $\bm{d}$ conforms to Definition~\ref{def:LTIbase}, such that  $\bm{\psi}=\bm{d}$ includes features satisfying condition \eqref{space_opt}. Nevertheless, it is still a nonlinear model after the initial transformation, while we look for a linear reconstruction map \eqref{eq:LTI:3} to construct a linear predictor for the nonlinear system.
To achieve the former, we need to allow the latent dynamical system to have a dimension $D\tsgn{\gg}d$, which cannot be achieved directly with diffeomorphisms since they preserve dimensionality. Therefore, we propose to lift the diffeomorphic features $\bm{d}$ to a higher dimensional space of monomials as this preserves the Koopman-invariance \eqref{space_opt} of $\bm{d}$.}
Thus, the idea is to define monomial coordinates based on the latent vector  $\bm{d}(\bm{x})\tsgn{=}\bm{y}\tsgn{=}[y_1,\dots,y_d]^{\top}$ through $y^{\bm{\alpha}}\tsgn{=}y_{1}^{\alpha_{1}} y_{2}^{\alpha_{2}} \cdots y_{d}^{\alpha_{d}}$, 
where $\bm{\alpha}\in\mathbb{N}_0^d$ is a multi-index. Then, we obtain a lifted coordinate vector by concatenating all monomials $y^{\bm{\alpha}}$ up to order $\|\bm{\alpha}\|_1\tsgn{=}\alpha_{1}\tsgn{+}\cdots\tsgn{+}\alpha_{d}\leq\overline{p}$ in a lexicographical ordering in a vector $\bm{y}^{[\overline{p}]}$. By construction, $\bm{y}^{[\overline{p}]}$ inherits the linear dynamical system description from $\bm{y}\tsgn{=}\bm{d}(\bm{x})$, as shown in the following lemma.
\begin{lemma}\label{lem2}
\clr{black}{For $\bm{y}^{[\bar{p}]}\!\!\in\!\!\mathbb{R}^{D}$ with $D\tsgn{=}((
d\tsgn{+}\overline{p})!/(d!\overline{p}!)) \tsgn{-} 1$, there exists a {transition} matrix $\bm{A}_{[\overline{p}]}(\underline{\bm{A}})$ describing the dynamics as a linear ordinary differential equation
$\dot{\bm{y}}^{[\overline{p}]}\tsgn{=}\bm{A}_{[\overline{p}]}(\underline{\bm{A}}) \bm{y}^{[\overline{p}]}$.}

\begin{proof}
By examining the dynamics of monomials corresponding to the multi-index $\bm{\alpha}\tsgn{=}[\alpha_1\tsgn{,}\dots\tsgn{,}\alpha_d]^{\top}$ with order $\|\bm{\alpha}\|_1\tsgn{=}p$, we obtain linear ODEs  \cite{Zeng2018} $\dot{\bm{y}}^{[p]}\tsgn{=}\bm{A}_{[{p}]}(\underline{\bm{A}})\bm{y}^{[p]}$,
    linearly dependent on $\underline{\bm{A}}$.
    Since all $\bm{y}^{[p]}$ systems $p\in\Set{N}$ are decoupled from each other, their concatenation up to order $\bar{p}$ as $\bm{\dot{y}}^{[\overline{p}]}\tsgn{=}\bm{A}_{[\overline{p}]}(\underline{\bm{A}})\bm{y}^{[\overline{p}]}$
    with $\bm{y}^{[\overline{p}]}\tsgn{=}[\bm{y}^{[1]^{\top}}\tsgn{,}\dots\tsgn{,}\bm{y}^{[p]^{\top}}\tsgn{,}\dots\tsgn{,}\bm{y}^{[\bar{p}]^{\top}}]^{\top}$ and $\bm{A}_{[\overline{p}]}(\underline{\bm{A}})\tsgn{=}\operatorname{diag}\{\bm{A}_{[1]}(\underline{\bm{A}})\tsgn{,}\dots\tsgn{,}\bm{A}_{[\overline{p}]}(\underline{\bm{A}})\} \in \mathbb{R}^{D\times D}$ remains a collection of linear ODEs. Given every $\bm{y}^{[p]}$ is a combination with replacement of $d$ elements and $p$ samples, the total amount of concatenated coordinates up to order $\overline{p}$ equals $D\tsgn{=}\sum^{\overline{p}}_{p\tsgn{=}1}\icol{p+d-1 \\
    p}{\tsgn{=}}\icol{d+\overline{p} \\
    d}\tsgn{-}1$
    via the "hockey-stick" identity~.
    This proves the concatenation up to order $\overline{p}$ leads to an extended system $\dot{\bm{y}}^{[\overline{p}]}\tsgn{=}\bm{A}_{[{p}]}(\underline{\bm{A}})\bm{y}^{[\overline{p}]}$ spanning invariant subspaces of $\dot{\bm{y}}\tsgn{=}\underline{\bm{A}}\bm{y}$ of size  $D$.
\end{proof}
\end{lemma}%
\looseness=-1
\begin{figure}[t!]
    \centering
  \begin{tikzpicture}[baseline= (a).base, ampersand replacement=\&]
    \node[scale=1.0] (a) at (0,0){
    	\begin{tikzcd}[column sep = huge, row sep = huge, ]
    	\mathbb{X} \arrow{r}{\color{magenta}\bm{d}} \arrow[swap]{d}{\bm{f}} \arrow{d}{\eqref{eq:sys}}  \arrow[thick, bend left]{rr}{\bm{\psi}} \arrow[thick,swap, bend left]{rr}{\text{Prop. \ref{prop1}}} \& \mathbb{Y} \arrow[swap]{d}{\color{magenta}\underline{\bm{A}}} \arrow{r}{\bm{\varrho}:~\bm{y} \mapsto \bm{y}^{[\overline{p}]}} \arrow[swap]{r}{\text{Lem. \ref{lem2}}} \& \mathbb{Z} 
    	\arrow[swap,thick]{d}{\bm{{\bm{A}}_{[\overline{p}]}}({\color{magenta} \underline{\bm{A}}})}\\%
    	\mathcal{T}_{\bm{x}}\mathbb{X} \arrow{r}{\frac{\partial \color{magenta}\bm{d}}{\partial \bm{x}}} \& \mathcal{T}_{\bm{y}}\mathbb{Y} \& \mathcal{T}_{\bm{z}}\mathbb{Z}
    	\arrow[thick,swap, bend left, dashrightarrow]{ll}{\color{magenta}\bm{C}}
    	\end{tikzcd}
    };
    \end{tikzpicture}    
    \vspace{-0.4cm}
    \caption{The diagram of our construction for learning a model of the form \eqref{eq:LTI:1}-\eqref{eq:LTI:3} with the construction pathway in bold and the maps to be learned in magenta. The sets $\Set{X},\Set{Y},\Set{Z}$ correspond to the immediate state-space, latent space and lifted linear model space, respectively; with corresponding tangent spaces denotes as $\mathcal{T}_{\bm{x}}\Set{X},\mathcal{T}_{\bm{y}}\Set{Y},\mathcal{T}_{\bm{z}}\Set{Z}$.
    }
    \label{fig:commDiag}
\end{figure}
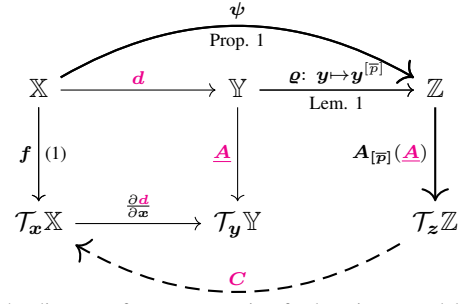
\looseness=-1
\clr{black}{Note that the matrix $\bm{A}_{[\overline{p}]}(\underline{\bm{A}})$ can be constructed as a block-diagonal concatenation of $\bm{A}_{[{p}]}({\bm{A}})$ matrices \cite{Zeng2018} up to order $\overline{p}$.}
Moreover, this monomial lifting of the latent linear system preserves the Koopman-invariance of the diffeomorphism $\bm{d}$ and satisfies \eqref{space_opt} as shown in the following proposition.

\begin{proposition}\label{prop1}
Assume the linear system $\dot{\bm{y}}=\underline{\bm{A}}\bm{y}$ is smoothly equivalent to system~\eqref{eq:sys} via a diffeomorphism $\bm{d}$ such that ${\bm{y}}=\bm{d}(\bm{x})$. Then the lifted features $\bm{\psi}=\bm{d}^{{[\overline{p}]}}$ satisfy \eqref{space_opt}, i.e.,  $\bm{\psi}(\bm{x})=\bm{d}^{{[\overline{p}]}}(\bm{x})=\bm{y}^{{[\overline{p}]}}$ are Koopman-invariant coordinates and define a latent linear system\looseness=-1
\begin{subequations}\label{eq:liftedLTI full}
\begin{align}
\bm{z}^{}_0 &=\bm{d}^{{[\overline{p}]}}(\bm{x}(0)), \label{eq:liftedLTI:1}\\
\dot{\bm{z}}^{} &=\bm{A}_{[\overline{p}]}(\underline{\bm{A}})\bm{z}^{}. \label{eq:liftedLTI:2}
\end{align}
\end{subequations}

\begin{proof}
{Consider vector fields $\dot{\bm{x}}=\bm{f}(\bm{x})$ and $\dot{\bm{y}}=\underline{\bm{A}}\bm{y}$ smoothly equivalent through a diffeomorphism $\bm{y}=\bm{d}(\bm{x})$.}
A simple chain of equalities
\begin{align}\label{prf2}
    \mathcal{G}_{\ko_{\bm{f}}}\bm{d}(\bm{x}) &\overset{Def. \ref{def:LTIbase}}{=}\bm{J}_{\bm{d}}(\bm{x})\bm{f}(\bm{x})\overset{Def. \ref{def:smthEQ}}{=}\underline{\bm{A}}\bm{d}(\bm{x})~=\mathcal{G}_{\ko_{\underline{\bm{A}}}}\bm{y} 
\end{align}
shows Koopman-invariant subspaces of the smoothly equivalent vector fields' infinitesimal generators evolve linearly with the same $\underline{\bm{A}}$.
Using the result of Lemma \ref{lem2}, $\bm{d}^{{[\overline{p}]}}(\bm{x})$ are Koopman-invariant coordinates of $\mathcal{G}_{\ko_{\bm{f}}}$ evolving linearly with $\bm{A}_{[{p}]}(\underline{\bm{A}})$ -- concluding the proof.
\end{proof}

\clr{black}{In essence, the above proposition establishes that the $\mathcal{G}_{\ko_{\bm{f}}}$- and $\mathcal{G}_{\ko_{\underline{\bm{A}}}}$-invariant coordinates have the same (linear) dynamics. This, in turn, means that diffeomorphically linearizable systems share the same spectra.}
\end{proposition}
\looseness=-1
\subsection{Parameterizing Stable System Matrices}
To simplify the computations involved for satisfying \eqref{stab_cnstr},
we propose to employ an unconstrained parameterization of stable matrices akin to \cite{Fan2021}.
As the latent dynamics is parameterized
in terms of low-dimensional matrices $\underline{\bm{A}}$, we utilize an unconstrained parameterization of all Hurwitz matrices $\underline{\bm{A}}$, described by the following lemma.
\begin{lemma}\label{lem:CThrwzParam}
Consider matrices ${\bm{N},\bm{Q},\bm{R} \in \Set{R}^{n \times n}}$, a positive constant $\epsilon\in\mathbb{R}_+$, and the matrix parameterization 
\begin{equation}\label{eq:AParam}
\underline{\bm{A}}(\bm{N}\!,\!\bm{Q}\!,\!\bm{R})\tsgn{=}(\bm{N} \bm{N}^{\top}\tsgn{+}\epsilon \bm{I})^{-1}(\tsgn{-}\bm{Q} \bm{Q}^{\top}\tsgn{-}\epsilon \bm{I}\tsgn{+}\frac{1}{2}(\bm{R}\tsgn{-}\bm{R}^{\top})).
\end{equation}
A matrix $\bm{\mathcal{A}}$ is Hurwitz if and only if $\exists$ $\bm{N},\bm{Q},\bm{R}, \epsilon$, such that $\bm{\mathcal{A}}\tsgn{=}\underline{\bm{{A}}}(\bm{N}\!,\!\bm{Q}\!,\!\bm{R})$.

\begin{proof}
\clr{black}{
Let $\bm{X}\tsgn{=}\bm{N} \bm{N}^{\top}\tsgn{+}\epsilon\bm{I}$, $\bm{Y}\tsgn{=}\bm{Q} \bm{Q}^{\top}\tsgn{+}\epsilon\bm{I}$ and $\bm{Z}\tsgn{=}\frac{1}{2}\left(\bm{R}\tsgn{-}\bm{R}^{\top}\right)$ such that $\bm{\mathcal{A}}\tsgn{=}\bm{X}^{-1}(\tsgn{-}\bm{Y}\tsgn{+}\bm{Z})$. \clr{black}{With $(\tsgn{-}\bm{Y}\tsgn{+}\bm{Z})$ {generalized negative-definite} \cite[Def. 2.1]{DUAN1998509} and $\bm{X}^{-1}$ symmetric positive-definite, the proof follows directly \cite[Thm. 3.1]{DUAN1998509} after transposing $\bm{\mathcal{A}}$.}}
\end{proof}
\end{lemma}
For $\epsilon\rightarrow 0^+$, the space of all Hurwitz matrices is covered.
Thus, we can optimize over the low-dimensional matrices $\bm{N}$, $\bm{Q}$ and $\bm{R}$ without worrying about the stability condition \eqref{stab_cnstr} as it is guaranteed by construction. Due to Lemma~\ref{lem2}, this stability extends to the lifted system matrix $\bm{A}_{[\overline{p}]}(\underline{\bm{A}}(\bm{N}\!,\!\bm{Q}\!,\!\bm{R}))$, such that we can reformulate the optimization problem \eqref{opt_prob} as shown in the following corollary.
\clr{black}{
\begin{corollary}\label{cor:singleConstr}
The minimizers $\hat{\bm{{N}}},\hat{\bm{Q}},\hat{\bm{R}}, \hat{\bm{C}}, \hat{\bm{\bm{d}}}(\cdot)$ of the optimization problem
\begin{subequations}\label{step2opt}
\begin{align}
\operatornamewithlimits{min}_{\substack{\bm{{N}},\bm{Q},\bm{R}, \bm{C}, \bm{\bm{d}}(\cdot)}}  &\sum^{N}_{i=1} {\left\|{{\bm{\xi}}}^{(i)}\tsgn{-}\begin{bmatrix}\bm{C}\bm{A}\\\bm{C}\end{bmatrix}\boldsymbol{\psi}(\bm{x}^{(i)})\right\|^{2}_2} \label{obj_opt_uncons} \\
\mathrm{s.t.} &  \quad 
{{\bm{\xi}}}_{1}^{(i)}\tsgn{=}\bm{J}^{-1}_{\bm{{d}}}\underline{\bm{A}}(\hat{\bm{N}}\tsgn{,}\hat{\bm{Q}}\tsgn{,}\hat{\bm{R}})\boldsymbol{d}(\bm{x}^{(i)})
\label{diffeo_cnstr}
\end{align}
\end{subequations}
defines a solution $\bm{\psi}=\hat{\bm{d}}^{{[\overline{p}]}}$, $\bm{A}=\bm{A}_{[\overline{p}]}(\underline{\bm{A}}(\hat{\bm{N}}\tsgn{,}\hat{\bm{Q}}\tsgn{,}\hat{\bm{R}}))$, $\bm{C}=\hat{\bm{C}}$ for the optimization problem~\eqref{opt_prob} and thereby define a model of the form \eqref{eq:LTI full}.

\begin{proof}
Using Proposition~\ref{prop1}, we can replace the condition \eqref{space_opt} by a \eqref{diffeo_cnstr} w.l.o.g. By Lemma \ref{lem:CThrwzParam}, the lower-rank matrix $\underline{\bm{A}}$ is Hurwitz by construction. Following Proposition \ref{prop1}, the eigenvalues of $\bm{A}_{[\bar{p}]}(\underline{\bm{A}})$ satisfy $\sum_{i=1}^{d} \alpha_{i} \underline{\lambda}_{i}$ with $\alpha_i \in \Set{N}_0$, making $\bm{A}_{[\overline{p}]}(\underline{\bm{A}})$ Hurwitz as well -- allowing us to replace the condition \eqref{stab_cnstr} with a Hurwitz matrix parameterization from Lemma \ref{lem:CThrwzParam}.
\end{proof}
\end{corollary}
\begin{remark}\label{simplification}
The above corollary allows for a twofold simplification of the original problem \eqref{opt_prob}.
Firstly, the unconstrained parameterization from Lemma \ref{lem:CThrwzParam} simplifies the optimization as the linear dynamics are Hurwitz by construction. Secondly, the above learning problem results in a learned nonlinear mapping and a linear dynamics matrix of immediate state-space dimension, w.l.o.g. for the considered system class, allowing us to parameterize a possibly high-dimensional \textit{linear predictor} with a function approximator and linear dynamics of comparatively low dimension.
\end{remark}}
\subsection{Structured Relaxation of Exact Smooth Equivalence}

To ease the use of standard training algorithms for expressive function approximators such as neural networks, we relax the optimization problem \eqref{step2opt} by considering \eqref{diffeo_cnstr} as an additional summand in the cost \eqref{obj_opt_uncons}.
{To allow for a structured relaxation of \eqref{step2opt}, one needs to ensure the function approximator $\bm{d}$ is of a suitable, \textit{diffeomorphic}, hypothesis class. For that, we side with the invertible neural network (INN) hypothesis class $\mathcal{H}_{_\text{INN}}$ as various INN architectures universally approximate $C^2$-diffeomorphisms ($D^2$) with respect to the $L^{p}$-/$\operatorname{sup}$-norm \cite{Teshima2020}. As $\mathcal{H}_{_\text{INN}}$ is dense in $D^2$, considering \eqref{diffeo_cnstr} as an additional cost summand admits an arbitrary small Koopman-invariance residual. \clr{black}{ Then, the following relaxation of \eqref{opt_prob} is possible}
\begin{subequations}\label{obj_opt_unconsFinal}
\clr{black}{\begin{align}
\operatornamewithlimits{min}_{\substack{{\bm{\theta}}_{\underline{\bm{A}}}\tsgn{,}\bm{C}\tsgn{,}\bm{\bm{d}} \in \mathcal{H}_{_{\text{INN}}}}}  \sum^{N}_{i=1}&{\left\|{{\bm{\xi}}}^{(i)}\tsgn{-}\begin{bmatrix}\bm{C}\bm{A}_{[\overline{p}]}\left(\underline{\bm{A}}({\bm{\theta}}_{\underline{\bm{A}}})\right)\\\bm{C}\end{bmatrix}{\bm{d}}^{{[\overline{p}]}}(\bm{x}^{(i)})\right\|^{2}_2} \label{eq:initOpt}\\
\tsgn{+}&\left\|{{\bm{\xi}}}_{1}^{(i)}\tsgn{-}\bm{J}^{-1}_{\bm{{d}}}\underline{\bm{A}}({\bm{\theta}}_{\underline{\bm{A}}})\boldsymbol{d}(\bm{x}^{(i)})\right\|^{2}_2\label{eq:pureSE}\\
\tsgn{+}&\left\|\bm{J_{{\bm{{d_{{}}}}}}}(\bm{0})\tsgn{-}\bm{I}\|_2^2 {+} \|\bm{d_{}}(\bm{0})\tsgn{-}\bm{0}\right\|_2^2\label{eq:loss_nearid} 
\end{align}}
\end{subequations}
where \eqref{eq:pureSE} replaces \eqref{diffeo_cnstr}, $\bm{\theta}_{\underline{\bm{A}}}=\{\bm{N}\!,\!\bm{Q}\!,\!\bm{R}\}$ and \eqref{eq:loss_nearid} is the near-identity enforcing cost \cite[Theorem 2.3]{Lan2013}. Although not necessary in principle, it improves convergence in a local neighborhood of the equilibrium.}

\clr{black}{\begin{remark}
The problem \eqref{obj_opt_unconsFinal} does not modify the structure of \eqref{obj_opt}, but merely adds the supervised cost \eqref{eq:pureSE} enforcing \eqref{diffeo_cnstr} and data independent cost terms \eqref{eq:loss_nearid} penalizing undesirable behavior. This results in an overall supervised learning approach to obtain a model of the form \eqref{eq:LTI full} by fitting $\bm{x}^{(i)} \mapsto {\bm{\xi}}^{(i)}\equiv\tilde{\bm{f}}(\bm{x}^{(i)})$.
\end{remark}}

For realizing complex diffeomorphisms, we employ INNs based on coupling flows (CF-INN) \cite{Dinh2017} that successively compose simpler diffeomorphisms called \emph{coupling layers} $\bm{{\hat{d}}_i}$ using the fact that diffeomorphic maps are closed under composition, so that $\bm{y}=\bm{{\hat{d}}}(\bm{x}) = \bm{{\hat{d}}_k} \circ ... \circ \bm{{\hat{d}}_1}(\bm{x})$.
Each coupling layer $\bm{{\hat{d}}_j}$ is defined to couple a disjoint partition of the input $\bm{x}=[\bm{x}^{\top}_a~\bm{x}^{\top}_b]^{\top}$ with two subspaces $\bm{x_a} \in \Set{R}^{d-n}$, $\bm{x_b} \in \Set{R}^n$ where $n \in \Set{N}$ and $d\geq2$, in a manner that ensures bijectivity. This can be realized via affine coupling flows (ACF), which have coupling layers\looseness=-1
\begin{equation}
\bm{{\hat{d}}}_{i}(\bm{x}^{(i)}) = \begin{bmatrix}\bm{x_a}^{(i)}\\\bm{x_b}^{(i)} \odot \operatorname{exp}(\bm{s}_j(\bm{x_a}^{(i)})) + \bm{t}_j(\bm{x_a}^{(i)})\end{bmatrix}
\label{eq:coupling_layer} 
\end{equation}
with scaling functions $\bm{s}_j:\mathbb{R}^n \mapsto \mathbb{R}^{N-n}$ and translation functions $\bm{t}_j:\mathbb{R}^n \mapsto \mathbb{R}^{N-n}$ that can be chosen freely. The parameters of the diffeomorphic learner consist of the weights and biases in the neural networks of the scaling and translation functions concatenated in parameters $\bm{w}=[\bm{w}^{\top}_{\bm{s}_1},\bm{w}^{\top}_{\bm{t}_1}, \cdots, \bm{w}^{\top}_{\bm{s}_k},\bm{w}^{\top}_{\bm{t}_k}]^{\top}$.

Since the ACF are constructed to be diffeomorphisms, we can optimize over the parameters $\bm{w}$ instead of diffeomorphisms in \eqref{obj_opt_unconsFinal}.
Crucially, it allows us to guarantee the stability of systems \eqref{eq:LTI full} induced by the solutions of \eqref{obj_opt_unconsFinal}, as shown in the following theorem. 
\begin{theorem}\label{thm:GAS}
Let diffeomorphisms $\bm{d}=\bm{{\hat{d}}_k} \circ \ldots \circ \,\bm{{\hat{d}}_1}(\bm{x})$ be parameterized through coupling layers  \eqref{eq:coupling_layer}, which are defined using continuously differentiable functions $\bm{s}_i$, $\bm{t}_i$.
Then, every optimization problem \eqref{obj_opt_unconsFinal} has a solution and yields a stable system~\eqref{eq:LTI full}.\looseness=-1
\end{theorem}
\begin{proof}
With $\underline{\bm{A}}$ Hurwitz by construction due to Lemma~\ref{lem:CThrwzParam}, following Proposition \ref{prop1},  the eigenvalues of $\bm{A}_{[\bar{p}]}(\underline{\bm{A}})$ are linear combinations $\sum_{i=1}^{d} \alpha_{i} \underline{\lambda}_{i}$ of multi-index entries $\alpha_i \in \Set{N}_0$, making $\bm{A}_{[\overline{p}]}(\underline{\bm{A}})$ Hurwitz as well.
By Proposition \ref{prop1} the map $\bm{\varrho}:\bm{y} \mapsto \bm{y}^{[\overline{p}]}$, representing a monomial basis of the argument, is an immersion as $\operatorname{rank}(\bm{J}_{\bm{\varrho}}(\bm{y}))=\operatorname{dim}(\bm{y})$. As ACFs are diffeomorphisms per construction with $C^1$ function approximators $\bm{s}_j$ and $\bm{t}_j$, $\bm{d}$ is an immersion by construction. Its composition with the lifting map $\bm{\varrho} \circ \bm{d}:\bm{y} \mapsto \bm{d}^{[\overline{p}]}$ is as well due to immersions being invariant under composition. With $\bm{A}_{[\overline{p}]}(\underline{\bm{A}})$ Hurwitz and $\bm{d}^{{[\overline{p}]}}$ immersible, the asymptotic stability of the lifted model \eqref{eq:LTI:1}-\eqref{eq:LTI:2} follows via \cite[Propositon 1]{Yi2021}. Hence, every optimization problem \eqref{obj_opt_unconsFinal} is guaranteed to yield an asymptotically stable system.
\end{proof}
This theorem allows to efficiently obtain approximate solutions to the optimization problem \eqref{opt_prob} in practice, since the differentiability condition for $\bm{s}_j$, $\bm{t}_j$ can be easily satisfied using neural networks with smooth activation functions. Therefore, it transforms the practically intractable problem~\eqref{opt_prob} into an easily implementable supervised learning problem. {
\begin{remark}
While we cannot ensure that the result of \eqref{obj_opt_unconsFinal} delivers exact solutions to \eqref{opt_prob}, Theorem~\ref{thm:GAS} guarantees that these solutions always yield stable systems \eqref{eq:LTI full}.
\end{remark}
}
\section{EVALUATION}\label{sec:Eval}
For the evaluation of the proposed Koopmanizing Flows, we compare its performance to the related work of SKEL\footnote{\scriptsize\url{https://github.com/fletchf/skel.git}} \cite{Fan2021} on the real-world LASA\footnote{\scriptsize\url{https://cs.stanford.edu/people/khansari/download.html}} handwriting dataset \cite{Khansari2011}, \clr{black}{commonly used to compare approaches for learning stable dynamical systems}. \clr{black}{The dataset consists of 26 human-drawn trajectories of various letters and shapes with 7 demonstrations each. We asses the methods' performance in both pure imitation and validation. For imitation, we train with data from every demonstration and test their reproduction. For validation, the LASA dataset is split in four training and three validation trajectories. Model-selection is performed on the training RMSE. The models for each shape are scaled to the range $[-1,1]^d$.}
 To ensure a fair comparison we use the same base sampling for both methods. For Koopmanizing Flows we sample 900 data points for each demonstration trajectory of the dataset. The inputs and targets are the 2D position and velocity, respectively.
\clr{black}{For imitation, 7 layer ACF, with scaling and translation being neural networks consisting of 3 hidden layers with 120 neurons, are used to learn the diffeomorphisms. For validation, the diffeomorphic learners consist of a 9 layer ACF and simpler transformation networks with 2 hidden layers and 50 neurons. To improve numerical stability, a dimension-wise $\operatorname{tanh}(\cdot)$ is taken as a final coupling layer -- scaling the latent space to the unit-box $\mathbb{Y} \subset [-1,1]^d$. Each ACF network has smooth Exponential Linear Units (ELU) as activation functions.} The dimension of the lifting coordinates is $D=44$ ($\bar{p}=8$), \clr{black}{respectively}. Learning is performed employing the ADAM optimizer \cite{Kingma2015AdamAM}. 
The resulting imitation trajectories displayed in Figure~\ref{fig:lasanearID} are simulated until five times the demonstration time, with a different coloring from the point when the demonstration time is exceeded. 
Due to the deterministic nature of Koopmanizing Flows, cross-sections of the demonstrations lead to a contraction to a mean trajectory when reproduced.
In order to evaluate our performance with respect to the related work SKEL \cite{Fan2021}, we compare position predictions in terms of dynamic time warping distance (DTWD) \cite{Berndt1994}, root mean squared error
(RMSE) and partial curve matching (PCM) \cite{Jekel2019}. Statistics for each of the frameworks \clr{black}{on both the imitation and validation task} are visualized in Figure~\ref{fig:frog1}. In comparison to SKEL, Koopmanizing Flows show superior performance in all metrics -- especially ones strongly related to the accuracy of the state-space geometry such as DTWD and PCM. This is to be expected as Koopmanizing Flows are geared towards directly learning features \clr{black}{allowing the system's evolution to be described as a superposition of (lifted) trajectories, or, more formally speaking:} that lie in the span of (generalized) eigenfunctions and inherently describe the state-space geometry \cite{Mezic2019}.
Furthermore, the performance of SKEL deteriorates in the long-term due to a nonlinear reconstruction map, which can cause the equilibrium point of the model to not lie at the end-point of the demonstrated movement. Therefore, the proposed Koopmanizing Flows show how a theoretically well-founded construction of lifting features results in superior performance.\\
\clr{black}{\textit{Discussion:} Regarding the function approximators, other diffeomorphic learners could be explored for an additional performance gain. In Figure \ref{fig:frog1}, the validation results suggest that both approaches are susceptible to outliers. This can be explained by an over-fitting tendency of linear reconstruction in case of validation. Options to resolve the aforementioned include considering various regularization techniques, e.g., sampling different data for reconstruction \cite{Korda2020b}.}
\section{CONCLUSION}
\begin{figure}
    \centering
    \includegraphics[width=0.99\linewidth]{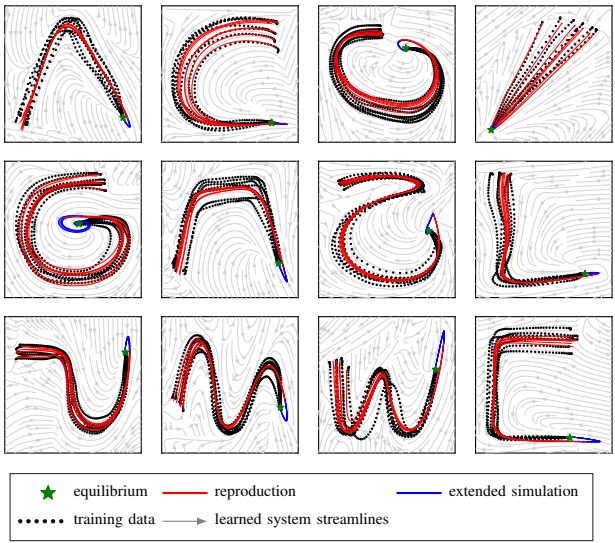}
    \begin{tikzpicture} 
    \pgfdeclareplotmark{mystar}{
    \node[star,star point ratio=2.25,minimum size=1pt, rounded corners=0.1,
          inner sep=0pt,draw=black!50!green,solid,fill=black!50!green] {};
}
    \begin{axis}[%
    hide axis,
    xmin=0,
    xmax=2,
    ymin=0,
    ymax=0.4,
    domain=0.1:0.1,
    legend style={draw=white!15!black,legend cell align=left,nodes={scale=1.2}},
    legend columns=3,
    width=0.5\columnwidth,
    height=0.5\columnwidth
    ]
    \addlegendimage{,mark=mystar, only marks}
    \addlegendentry{equilibrium};
    \addlegendimage{red, thick}
    \addlegendentry{reproduction};
    \addlegendimage{blue, thick}
    \addlegendentry{extended simulation};
    \addlegendimage{black, dots, ultra thick}
    \addlegendentry{training data};
    \addlegendimage{gray,-latex}
    \addlegendentry{learned system streamlines};
    \end{axis}
\end{tikzpicture}
\vspace{-0.75em}
    \caption{Koopmanizing Flows yield trajectories similar in shape to the real ones, demonstrating the lifting construction captures the geometry of the original state-space.  
    }
    \label{fig:lasanearID}
\end{figure}
\definecolor{grey}{HTML}{cccccc}
\begin{figure}[t!]
\centering
\begin{tikzpicture}
    \begin{axis}[%
    hide axis,
    xmin=0,
    xmax=2,
    ymin=0,
    ymax=0.4,
    domain=0.1:0.1,
    legend style={draw=white!15!black,legend cell align=left,nodes={scale=1.2}},
    legend columns=4,
    width=0.5\columnwidth,
    height=0.5\columnwidth
    ]
    \addlegendimage{black, mark=square, only marks}
    \addlegendentry{Imitation};
    \addlegendimage{grey, mark=square*, only marks}
    \addlegendentry{Validation};
    \addlegendimage{blue, mark=diamond*, only marks}
    \addlegendentry{SKEL \cite{Fan2021}};
    \addlegendimage{red, mark=*, only marks}
    \addlegendentry{Koopmanizing Flows};
    \end{axis}
\end{tikzpicture}
\includegraphics[width=0.94\linewidth]{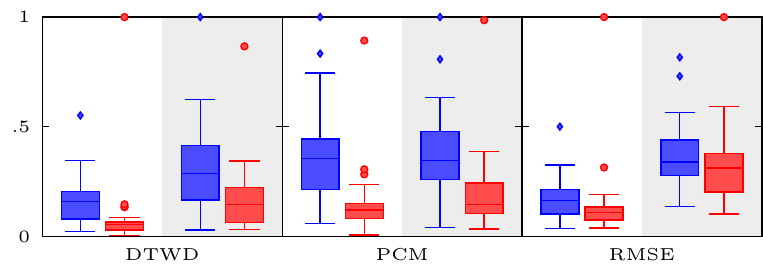}
\vspace{-0.75em}
\caption{\label{fig:frog1}Koopmanizing Flows show superior performance, especially in capturing the accuracy of the shapes. Each metric is normalized to lie in the range $[0,1]$ for ease of comparison.
}
\end{figure}
We have presented Koopmanizing Flows -- a novel, theoretically well-founded and fully data-driven learning framework for stable Koopman operator models with linear prediction and reconstruction.
Our results demonstrate improved performance compared to related work with nonlinear original state reconstruction even though employing the more practicable linear reconstruction. An experimental evaluation on the LASA benchmark shows the superior efficacy of our principled learning approach.
\vspace{-0.5em}
\bibliographystyle{IEEEtran}
\bibliography{bibfile}
\end{document}